\documentclass{article}

\usepackage{hyperref}
\usepackage{microtype}
\usepackage{custom_commands}
\usepackagewithoutnatbib[accepted]{icml2020}
\usepackage{custom_citations}

\usepackage{xpatch}
\xpatchbibmacro{cite}{\printnames{labelname}}{\printtext[bibhyperref]{\printnames{labelname}}}{}{}
\xpatchbibmacro{textcite}{\printnames{labelname}}{\printtext[bibhyperref]{\printnames{labelname}}}{}{}

\bibliography{references.bib}

\usepackage{amsthm} 
\newtheorem{theorem}{Theorem}
\newtheorem{proposition}[theorem]{Proposition}
\newtheorem{definition}[theorem]{Definition}
\newtheorem{corollary}[theorem]{Corollary}

\usepackage{nicefrac}
\usepackage{fontawesome}

\newcommand{\test}{\ensuremath{*}}
\newcommand{\METHOD}{decoupled sampling}
\newcommand{\dsgp}{DSGP}
\newcommand{\rff}{RFF}
\newcommand{\numBasisTotal}{b}
\usepackage{caption}

\let\UrlFont\rmfamily

\makeatletter

\renewcommand{\icmlcorrespondingauthor}[1]{
\ifdefined\isaccepted
 \ifdefined\icmlcorrespondingauthor@text
   \g@addto@macro\icmlcorrespondingauthor@text{#1}
 \else
   \gdef\icmlcorrespondingauthor@text{#1}
 \fi
}

\makeatother

\begin{document}

\twocolumn[
\icmltitle{Efficiently Sampling Functions from Gaussian Process Posteriors}

\icmlsetsymbol{equal}{*}
\begin{icmlauthorlist}
\icmlauthor{James T. Wilson}{equal,icl}
\icmlauthor{Viacheslav Borovitskiy}{equal,spbu,pdmi}
\icmlauthor{Alexander Terenin}{equal,icl}
\\
\icmlauthor{Peter Mostowsky}{equal,spbu}
\icmlauthor{Marc Peter Deisenroth}{ucl}
\end{icmlauthorlist}

\icmlaffiliation{icl}{Imperial College London}
\icmlaffiliation{ucl}{University College London}
\icmlaffiliation{spbu}{St. Petersburg State University}
\icmlaffiliation{pdmi}{St. Petersburg Department of Steklov Mathematical Institute of Russian Academy of Sciences} 

\icmlcorrespondingauthor{
\email{j.wilson17@imperial.ac.uk}, \quad
\email{viacheslav.borovitskiy@gmail.com},
\email{a.terenin17@imperial.ac.uk},
and \email{pmostowsky@gmail.com}}
\strut
]

\printAffiliationsAndNotice{\icmlEqualContribution}

\begin{abstract}
Gaussian processes are the gold standard for many real-world modeling problems, especially in cases where a model's success hinges upon its ability to faithfully represent predictive uncertainty. These problems typically exist as parts of larger frameworks, wherein quantities of interest are ultimately defined by integrating over posterior distributions. These quantities are frequently intractable, motivating the use of Monte Carlo methods. Despite substantial progress in scaling up Gaussian processes to large training sets, methods for accurately generating draws from their posterior distributions still scale cubically in the number of test locations. We identify a decomposition of Gaussian processes that naturally lends itself to scalable sampling by separating out the prior from the data. Building off of this factorization, we propose an easy-to-use and general-purpose approach for fast posterior sampling, which seamlessly pairs with sparse approximations to afford scalability both during training and at test time. In a series of experiments designed to test competing sampling schemes' statistical properties and practical ramifications, we demonstrate how \emph{decoupled sample paths} accurately represent Gaussian process posteriors at a fraction of the usual cost.
\end{abstract}

\section{Introduction}
Gaussian processes (GPs) are a powerful framework for reasoning about unknown functions $f$ given partial knowledge of their behaviors. In decision-making scenarios, well-calibrated predictive uncertainty is crucial for balancing important tradeoffs, such as exploration versus exploitation and long-term versus short-term rewards.
Bayesian methods naturally strike this balance \cite{ghavamzadeh2015bayesian, shahriari2015taking}.
While many quantities of interest defined with respect to Bayesian posteriors cannot be computed analytically (such as expectations of nonlinear functionals), they may be readily estimated via Monte Carlo methods.
Depending on this sample-based estimator's relative cost and statistical behavior, its performance may vary from state-of-the-art to method-of-last-resort.

Unlike methods for scalable training \cite{hensman13,wang2019exact}, techniques for efficiently sampling from GP posteriors have received relatively little attention in the machine learning literature. 
On the one hand, na\"{i}ve approaches to sampling are statistically well-behaved, but scale poorly owing to their need to solve for increasingly large linear systems at test time.
On the other hand, fast approximation strategies using Fourier features \cite{rahimi08} avoid costly matrix operations, but are prone to misrepresenting predictive posteriors \cite{wang18,mutny18,calandriello19}. Investigating their respective behaviors, we find that many of these strategies are complementary, with one often excelling where others falter. 
Motivated by this comparison of strengths and weaknesses, we leverage a lesser known decomposition of GP posteriors that allows us to incorporate the best of both worlds.

Our approach centers on the observation that we may implicitly condition Gaussian random variables by combining them with an explicit corrective term.
Translating this intuition to GPs, we may decompose the posterior as the sum of a prior and an update. By doing so, we are able to separately represent each of these terms using a basis well-suited for sampling.
This notion of "conditioning by kriging" was first presented by Matheron in the early 1970s, with various applications to geostatistics \cite{journel1978mining, defouquet94, chiles09}.
The concept was later rediscovered in astrophysics \cite{hoffman91, vandeweygaert96}, where it has been used to help simulate the universe as we know it.

We unite these ideas with techniques from the growing literature on approximate GPs to obtain an easy-to-use and general-purpose approach for accurately sampling from GP posteriors in linear time.

\section{Review of Gaussian processes}
\label{sec:background}

\begin{figure*}[t]
\center
\includegraphics[width=\textwidth]{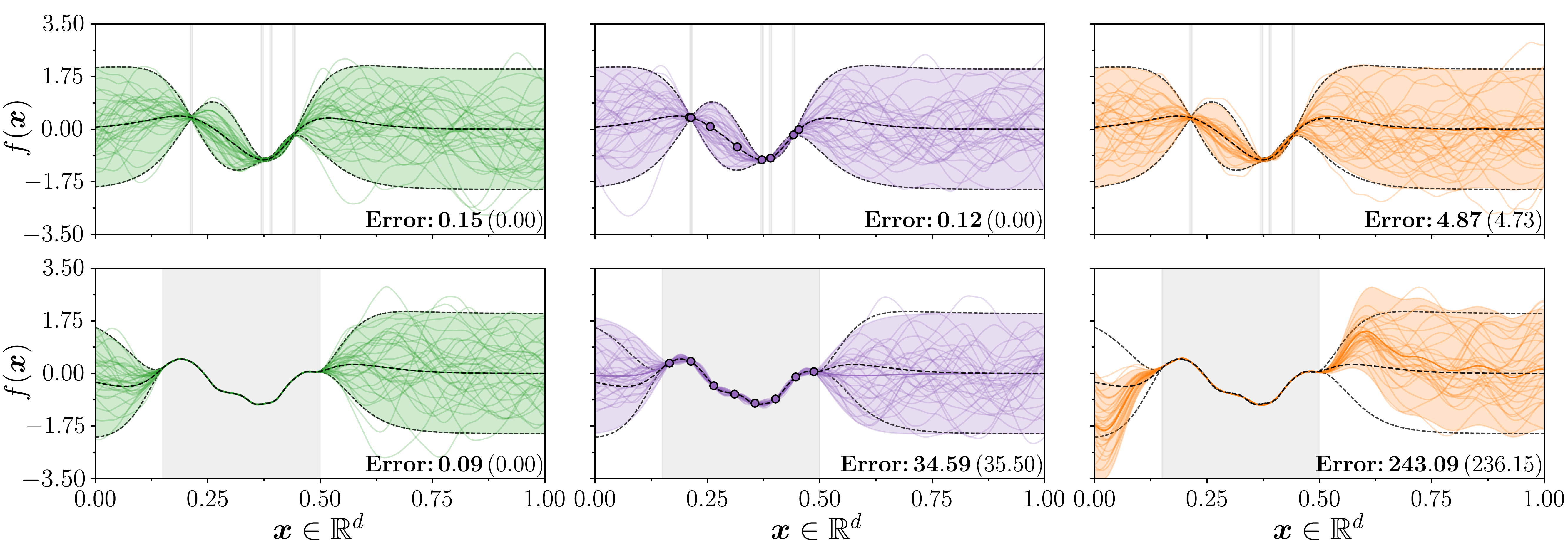}
\vspace{-18pt}
\caption{Comparison of GP posteriors and sample paths given $n=4$ (top) and $n=1000$ (bottom) observations at shaded locations. Error values shown in bottom right-hand corner of each figure denote 2-Wasserstein distances (see Section~\ref{sec:experiments}) between empirical (closed-form) posteriors and the true posterior (dashed black). \emph{Left:} Mean and two standard deviations of exact GP posterior (green) along with samples at $\test = 1024$ test locations. \emph{Middle:} Sparse GP with inducing variables $\v{u}$ at $m=8$ locations $\v{z} \in \c{X}$ denoted by `$\circ$'. \emph{Right:} Random Fourier feature-based GP with $\ell = 2000$ basis functions; for $n=1000$, variance starvation has started to set in and predictions away from the data show visible signs of deterioration.
}
\label{fig:comparison_of_posteriors}
\vspace{-6pt}
\end{figure*}

As notation, let $f : \c{X} \to \mathbb{R}$ be an unknown function with domain $\c{X} \subseteq \mathbb{R}^{d}$ whose behavior is indicated by a training set consisting of $n$ Gaussian observations $y_{i} = f(\v{x}_{i}) + \eps_i$ subject to measurement noise $\eps_i \sim \c{N}(0, \sigma^{2})$.

A Gaussian process is a random function $f : \c{X} \to \R$ such that, for any finite set of locations $\m{X}_{\test} \subseteq \c{X}$, the random vector $\v{f}_{\test} = f(\m{X}_{\test})$ follows a Gaussian distribution. In particular, if $f \sim \c{GP}(\mu, k)$, then $\v{f}_{\test} \sim \c{N}(\v{\mu}_{\test}, \m{K}_{\test,\test})$ is multivariate normal with covariance $\m{K}_{\test,\test} = k(\m{X}_{\test},\m{X}_{\test})$ specified by a kernel $k$. Henceforth, we assume a zero-mean prior $\mu(\cdot) = 0$ and continuous, stationary covariance function $k(\v{x}, \v{x}^\prime) = k(\v{x} - \v{x}^\prime)$.

Given $n$ observations $\v{y}$, the GP posterior at $\m{X}_{\test}$ is defined as $\v{f}_{\test} \given \v{y} \sim \c{N}(\v{m}_{\test \given n}, \m{K}_{\test,\test \given n})$, where we have defined
\[
\label{eqn:gp_posterior_moments}
\begin{aligned}
\v{m}_{\test\given n}
    &= \m{K}_{\test,n}(\m{K}_{n,n} + \sigma^{2}\m{I})^{-1}\v{y}
\\
\m{K}_{\test,\test\given n}
    &=\m{K}_{\test,\test} - \m{K}_{\test,n}(\m{K}_{n,n} + \sigma^{2}\m{I})^{-1}\m{K}_{n,\test}.
\end{aligned}
\]

When using \eqref{eqn:gp_posterior_moments} to help guide reinforcement learning agents \cite{kuss2004gaussian}, black-box optimizers \cite{snoek2012practical}, and other complex algorithms, we often rely on samples to estimate quantities of interest. The standard way of generating these samples is via a location-scale transform of Gaussian random variables $\v{\zeta} \sim \c{N}(\v{0}, \m{I})$, namely
\[
\label{eqn:sampler_nonparametric_affine}
\v{f}_{\test} \given \v{y}
    =
    \v{m}_{\test \given n}^{\vphantom{-1}} + \m{K}_{\test,\test \given n}^{1/2} \v{\zeta},
\]
where $(\cdot)^{1/2}$ denotes a matrix square root, such as a Cholesky factor.
Since this scheme is exact up to numerical error, we take it to be the gold standard against which the sample quality of alternatives will be judged. 
Unfortunately, this sampling strategy is also one of the least scalable, since the cost of computing $\m{K}_{\test,\test \given n}^{1/2}$ is already $\c{O}(*^{3})$.

The first column of Figure~\ref{fig:comparison_of_posteriors} visualizes sampling from a GP posterior given varying amounts of training data $n$. 
Since matrices on the right-hand-side of \eqref{eqn:gp_posterior_moments} grow as training sets increases in size, this method of sampling can be seen to accumulate little to no error as $n$ increases.
However, this growth requires us to invert increasingly large matrices both during training and at test time, causing standard GP inference and sampling methods to scale poorly in $n$.

\subsection{Function-space approximations to GPs}
\label{sec:function-space-approx}

The preceding interpretation of GPs, as distributions over functions with Gaussian marginals, is commonly known as the \emph{function-space} view \cite{rasmussen06}.
From this perspective, a natural way of approximating GPs is to represent $f$ in terms of its behavior $\v{u} = f(\m{Z})$ at a carefully chosen set of \emph{inducing locations} $\m{Z} = \{\v{z}_1,...,\v{z}_m\}$. 
In line with this function-space intuition of reasoning about $f$ via a small set of locations, this family of approximations is commonly referred to as \emph{sparse Gaussian processes}.

Rather than directly conditioning on observations $\v{y}$, sparse GPs begin by defining \emph{inducing distributions} $q(\v{u})$ that explain for the data.
Over the years, distinct iterations of sparse GPs have proposed different inducing paradigms \cite{snelson06, titsias09a, hensman17}.
In this work, we remain agnostic regarding the choice of $q(\v{u})$ and simply assume access to draws $\v{u} \sim q(\v{u})$.

Given $q(\v{u})$, we approximate posterior distributions as
\[
p(\v{f}_*\given\v{y}) \approx \int_{\R^m} p(\v{f}_*\given\v{u})q(\v{u}) \d\v{u}.
\]
If $\v{u} \sim \c{N}(\v\mu_{\v{u}}, \m\Sigma_{\v{u}})$, we compute this integral analytically to obtain a Gaussian distribution with mean and covariance
\[
\begin{aligned}
\v{m}_{\test \given m}
    &= \m{K}_{\test, m}^{\vphantom{-1}}\m{K}_{m,m}^{-1}\v{\mu}_{m}^{\vphantom{-1}}
    \\
\mathllap{\m{K}}_{\test, \test \given m}
    &= \m{K}_{\test, \test}^{\vphantom{-1}}
        \!\!+\! \m{K}_{\test, m}^{\vphantom{-1}}\m{K}_{m,m}^{-1}
        (\m{\Sigma}_{\v{u}}^{\vphantom{-1}} \!\!-\!\m{K}_{m,m}^{\vphantom{-1}})
        \m{K}_{m,m}^{-1}\mathrlap{\m{K}_{m,\test}^{\vphantom{-1}}.}
\quad
\end{aligned}
\label{eqn:svgp_posterior_moments}
\]
By virtue of explaining for $n$ observations using $m$ inducing variables, sparse GPs can be trained with $\c{O}(\tilde{n} m^{2})$ time complexity, where the choice of batch size $1 \le \tilde{n} \le n$ depends on the particular algorithm. Since high-quality approximations can be constructed using $m \ll n$ \cite{burt19}, sparse GPs drastically improve upon their exact counterparts' $\c{O}(n^{3})$ scaling.

While posterior moments \eqref{eqn:svgp_posterior_moments} may be computed at reduced cost, this benefit does not carry over when sampling. The standard procedure for sampling from sparse GPs is the same as in \eqref{eqn:sampler_nonparametric_affine} and incurs $\c{O}(*^{3})$ cost.\footnote{Select inducing point methods allow fast sampling from degenerate posteriors, see \textcite{quinonero05}.} When used to drive Monte-Carlo-based algorithms, sparse GPs can therefore be fast during training but slow during deployment. The middle column of Figure~\ref{fig:comparison_of_posteriors} depicts samples from a sparse GP posterior with $m=8$ inducing locations.

\subsection{Weight-space approximations to GPs}
\label{sec:weight-space-approx}

In the function-space view of GPs, we reason about $f$ in terms of the values it may assume at locations $\v{x} \in \c{X}$.
We now turn to the \emph{weight-space} view, where we will reason about $f$ as a weighted sum of basis functions.
Per the kernel trick \cite{scholkopf01}, $k$ can be viewed as the inner product in a reproducing kernel Hilbert space (RKHS) $\c{H}$ equipped with a feature map $\varphi: \c{X} \to \c{H}$.
If $\c{H}$ is separable, we may approximate this inner product as
\[
k(\v{x}, \v{x}^\prime) = \innerprod{\varphi(\v{x})}{\varphi(\v{x}^\prime)}_\c{H} \approx \v\phi(\v{x})^\top \v\phi(\v{x}),
\]
where $\v\phi: \c{X} \to \R^\ell$ is a finite-dimensional feature map \cite{rasmussen06}. For stationary covariance functions, Bochner's theorem implies that a suitable $\ell$-dimensional feature map can be constructed via a set of \emph{random Fourier features} (\rff) \cite{rahimi08}.

In this case, we have
$\phi_{i}(\v{x}) = \sqrt{\nicefrac{2}{\ell}}\,\cos(\v{\theta}_{i}^{\top}\v{x} + \tau_{i})$, where $\v{\theta}_{i}$ are sampled proportional to the kernel's spectral density and $\tau_{j} \sim U(0, 2\pi)$.
By defining the \emph{Bayesian linear model}
\<
\label{eqn:bayes_linear}
f(\cdot) &= \sum_{\smash{i=1}}^{\smash{\ell}} w_i \phi_i(\cdot)
&
w_i &\sim \c{N}(0, 1),
\>
we obtain an $\ell$-dimensional GP approximation. As in previous sections, $f$ is now a random function with Gaussian marginals. However, this stochasticity is now entirely controlled by the distribution of weights $\v{w}$.

For Gaussian likelihoods, the posterior weight distribution $\v{w} \given \v{y} \sim \c{N}(\v{\mu}_{\v{w} \given n}, \m{\Sigma}_{\v{w} \given n})$ is Gaussian with moments
\[
\begin{aligned}
\v\mu_{\v{w}\given n}
    &= (\m{\Phi}^{\top}\m{\Phi} + \sigma^{2}\m{I})^{-1}
    \m{\Phi}^{\top}\v{y}
    \\
\m{\Sigma}_{\v{w} \given n}
    &= (\m{\Phi}^{\top}\m{\Phi} + \sigma^{2}\m{I})^{-1} \sigma^{2},
\end{aligned}
\label{eqn:blm_posterior_moments}
\]
where $\m{\Phi} = \v{\phi}(\m{X})$ is an $n \times \ell$ feature matrix.
In both cases, we may solve for the right-hand side at $\c{O}(\min\{\ell, n\}^{3})$ cost by applying the Woodbury matrix identity.

Approximating the posterior $f \given \v{y}$ as weighted sums of basis functions in \eqref{eqn:bayes_linear} is particularly advantageous for purposes of sampling. 
As before, we may generate draws from \eqref{eqn:blm_posterior_moments} by first computing $\m{\Sigma}_{\v{w} \given n}^{1/2}$ at $\c{O}(\ell^{3})$ cost.\footnote{Alternatively, we may generate draws at $\c{O}(n^{3})$ cost by instead utilizing an eigen-decomposition \cite{seeger08}.}
Unlike before, we now sample weight vectors rather than function values and each draw now defines an actual \emph{function} evaluable at arbitrary locations $\v{x} \in \c{X}$. These methods have recently attracted attention in Bayesian optimization \cite{hernandez2014predictive, shahriari2015taking}, where the ability to fine-tune test locations $\m{X}_{*}$ by differentiating through samples is particularly valuable \cite{wilson2018maximizing}.

Unfortunately, these efficiency gains are counterbalanced by loss in expressivity. GP approximations equipped with covariance functions arising from finite-dimensional feature maps are well-known to exhibit undesirable pathologies at test time, see \textcite{rasmussen05}. 
In the case of Fourier-feature-based approximations, this tendency manifests as \emph{variance starvation}, whereby their extrapolatory predictions become increasingly ill-behaved as $n$ increases \cite{wang18,mutny18,calandriello19}. 
Intuitively, this occurs because the Fourier basis is only efficient at representing stationary GPs. The posterior, however, is generally nonstationary.
This tendency is evident in the right column of Figure~\ref{fig:comparison_of_posteriors}:  samples from the posterior clearly deteriorate in quality as we transition from low to high-data regimes.

\paragraph{Motivation.}
Prior to presenting our primary contributions, we briefly pause to restate key trends discussed above and shown in Figure~\ref{fig:comparison_of_posteriors}. 
Sampling from sparse GPs accommodates large amounts of training data $n = \vert\m{X}\vert$, but scales poorly with the number of test locations $\test = \vert\m{X}_{*}\vert$. 
Conversely, sampling from Fourier-feature-based weight-space approximations scales gracefully with $\test$, but results in high approximation error as $n$ increases. 
Function- and weight-space approaches to sampling from GP posteriors therefore exhibit opposing strengths and weaknesses.

Hence, the question: \emph{can we obtain the best of both worlds?}

\section{Sampling with Matheron's rule}
\label{sec:methods}

\begin{figure*}[t]
\center
\includegraphics[width=\textwidth]{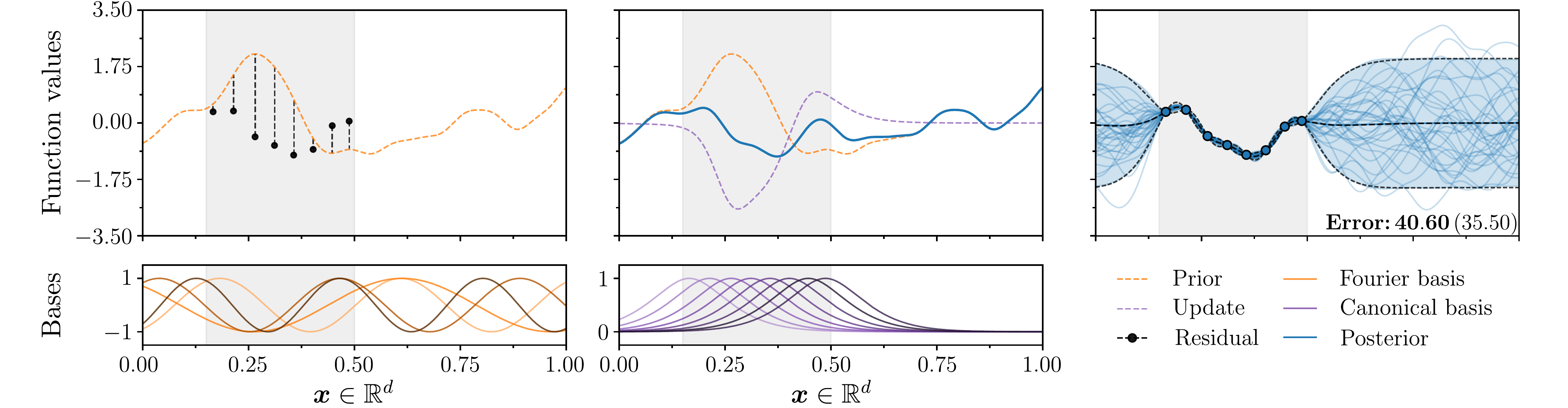}
\caption{Visual overview of \METHOD{} with a weight-space prior (orange) and a function-space update (purple); this example continues from Figure~\ref{fig:comparison_of_posteriors}. \emph{Left:} 1000 Fourier basis functions $\phi_{i}(\v{x}) = \cos(\v{\theta}_{i}^{\top}\v{x} + \tau_{i})$ are used to construct a function draw $f(\cdot) = \v{\phi}(\cdot)^{\top}\v{w}$ from an approximate prior (orange), resulting in residuals (dashed black) at each of $m = 8$ inducing locations $\v{z}_{j} \in \m{Z}$ (black circles). 
\emph{Middle:} a conditional sample path $f \given \v{u}$ (blue) is formed by adding an update (purple) consisting of canonical basis functions~$\psi_{j}(\cdot) = k(\cdot, \v{z}_{j})$ to $f$. \emph{Right:} the empirical distribution of sample paths $f \given \v{u}$ is compared with that of the sparse GP posterior~(dashed black). 2-Wasserstein errors of empirical (closed-form) posteriors were measured against the exact GP's moments.}
\label{fig:matheron_overview}
\end{figure*}

Our approach to designing an improved sampling scheme, which doubles as a rough outline for this section, is as follows: (i) analyze the shortcomings of existing methods, (ii) identify a decomposition of GPs that isolates these issues, (iii) represent the different terms using bases that address their corresponding issues. 
We begin by reviewing \emph{Matheron's rule} for Gaussian random variables \cite{journel1978mining, chiles09, doucet10}, which is central to our analysis.

\begin{theorem}[Matheron's Rule]
\label{thm:matheron_finitedim}
Let $\v{a}$ and $\v{b}$ be jointly Gaussian random variables.
Then the random variable $\v{a}$ conditional on $\v{b} = \v{\beta}$ is equal in distribution to
\[
(\v{a} \given \v{b} = \v\beta)
    \overset{\d}{=} 
    \v{a} + 
    \Cov(\v{a}, \v{b}) \Cov(\v{b}, \v{b})^{-1}(\v{\beta} - \v{b}).
    \label{eqn:matheron_finitedim}
\]
\end{theorem}
\begin{proof}
Follows immediately by computing the mean and covariance of both sides.
\end{proof}

Intuitively, Matheron's rule tells us that conditional random variable $\v{a} \given \v{b}$ can be broken up into a term representing the prior $p(\v{a}, \v{b})$ and a term that communicates the error in the prior upon observing that $\v{b} = \v{\beta}$. Hence, we may sample $\v{a} \,\given\, \v{b}$ by drawing $\v{a}$ and  $\v{b}$ together from the prior and, subsequently, updating $\v{a}$ to account for residuals $\v{\beta} - \v{b}$ as in \eqref{eqn:matheron_finitedim}. 
The corresponding statement for GPs is as follows.

\begin{corollary}
\label{cor:matheron_gp}
For a Gaussian process $f \sim \c{GP}(0, k)$ with marginal $\v{f}_{m} = f(\m{Z})$, the process conditioned on $\v{f}_{m} = \v{u}$ admits, in distribution, the representation
\[
\label{eqn:matheron_gp}
\underbracket[0.5pt]{(f \given\v{u})(\cdot)\vphantom{\m{K}_{m,m}^{-1}}}_{\t{posterior}}
    \overset{\d}{=}
    \underbracket[0.5pt]{f(\cdot) \vphantom{\m{K}_{m,m}^{-1}}}_{\t{prior}} 
    + \underbracket[0.5pt]{k(\cdot, \m{Z})\m{K}_{m,m}^{-1}(\v{u} - \v{f}_{m})}_{\t{update}}.
\]
\end{corollary}
\begin{proof}
By Theorem \ref{thm:matheron_finitedim}, the corollary holds for arbitrary finite-dimensional marginals, so the claim follows.
\end{proof}

Unlike \eqref{eqn:gp_posterior_moments} and \eqref{eqn:svgp_posterior_moments}, Corollary~\ref{cor:matheron_gp} defines a \emph{pathwise update}: rather than conditioning the prior as a distribution, we update the prior as realized in terms of sample paths. As we will soon see, this ability to go from prior to posterior (function) draws without needing to compute posterior covariance matrices (and their square-roots) will be the key to unlocking fast and accurate sampling from GP posteriors.

We are not the first to have realized this fact. This approach to simulating Gaussian conditionals is implicit in Matheron's pioneering work in the field of geostatistics, where it was subsequently popularized by \textcite{journel1978mining}. Decades later, \eqref{eqn:matheron_gp} was rediscovered in the context of $N$-body simulations by \textcite{hoffman91}. We combine these ideas with modern machine learning methods (such as sparse GPs and random Fourier features) to create a more efficient approach to sampling.

\subsection{Pathwise updates in weight- and function-spaces}
Rewriting the standard formulae for sparse and exact GP posteriors, respectively, as pathwise updates in accordance with Theorem~\ref{thm:matheron_finitedim}, we obtain
\begin{align}
\v{f}_{\test} \given\, \v{u}
    &\overset{\d}{=} 
        \v{f}_{\test} + \m{K}_{*,m}\m{K}_{m,m}^{-1}(\v{u} - \v{f}_{m})
    \label{eqn:matheron_gp_sparse}
\\
\v{f}_{\test} \given\, \v{y}
    &\overset{\d}{=} 
        \v{f}_{\test} + \m{K}_{*,n}(\m{K}_{n,n} + \sigma^{2}\m{I})^{-1}(\v{y}
        - \v{f} - \v{\varepsilon}).
    \label{eqn:matheron_gp_exact}  
\end{align}
When sampling from sparse GPs in \eqref{eqn:matheron_gp_sparse}, we draw $\v{f}_{*}$ and $\v{f}_{m}$ together from the prior, and independently generate target values $\v{u} \sim q(\v{u})$. 
When sampling from exact GPs in \eqref{eqn:matheron_gp_exact}, we again begin by jointly drawing $\v{f}_{*}$ and $\v{f}$ from the prior. 
Here however, we no longer need to generate targets $\v{u} = \v{y}$. Instead, we combine $\v{f}$ with noise variates $\v{\varepsilon} \sim \c{N}(\v{0}, \sigma^2 \m{I})$ such that $\v{f} + \v{\varepsilon}$ constitutes a draw from the prior distribution of $\v{y}$.

Turning to the weight-space setting, the analogous pathwise update given an initial weight vector $\v{w} \sim \c{N}(\v{0}, \m{I})$ is then
\[
\v{w} \given \v{y} 
    \overset{\d}{=} 
        \v{w} + \m{\Phi}^{\top}
        (\m{\Phi}\m{\Phi}^{\top} + \sigma^2\m{I})^{-1}
        (\v{y} - \m{\Phi}\v{w} - \v{\varepsilon}).
    \label{eqn:matheron_gp_rff}
\]
At first glance, it appears that sampling via Theorem~\ref{thm:matheron_finitedim} does not improve over standard methods.
Whereas \eqref{eqn:matheron_gp_rff} is of modest practical interest (it allows us to sample at $\c{O}(\min\{\ell, n\}^{3})$ cost without resorting to an eigen-decomposition), \eqref{eqn:matheron_gp_sparse} and \eqref{eqn:matheron_gp_exact} are actually more expensive than their standard counterparts.

At the same time, however, Theorem~\ref{thm:matheron_finitedim} allows us to view GP posteriors from a different perspective.
In particular, separating the effect of the prior from that of the data allows us to better diagnose the different sampling scheme's shortcomings.
For function-space approaches, we see that $\c{O}(\test^{3})$ time complexity is specific to the prior, since the update is linear in $\test$. 
For weight-space methods, we see that erratic extrapolations stem from difficulty representing the data (i.e., the update), since stationary priors are well-behaved under the Fourier basis.
Equipped with a better understanding of \emph{why} these methods fail, we now demonstrate \emph{how} to address these issues.

\subsection{Pathwise updates with decoupled bases}
\label{sec:decoupled_sampling}
So far, we have implicitly assumed a unified view of GP posteriors: when sampling in weight-space and in function-space, we sought to generate draws from conditional distributions over weight vectors and function values, respectively.
Several recent works \cite{cheng17, salimbeni18, shi19} have introduced decompositions that separately represent different aspects of GPs via different bases, such as RKHS subspaces and their orthogonal complements.  
There, the authors exploit the different bases' properties to better approximate the overarching process. 
We will do the same, but our goal will be to efficiently sample from the accompanying posteriors.

Corollary~\ref{cor:matheron_gp} is a pathwise update for Gaussian random variables that doubles as a decomposition of the posterior.
To further build on this distinction, we restate this using a weight-space approximation to the prior
\<
\label{eqn:decoupled_sample_path}
\underbracket[0.5pt]{
(f \given \v{u})(\cdot)\vphantom{\sum_{j=1}^{\smash{m}} v_{j} k(\cdot, \v{z}_{j})}
}_{\t{sparse posterior}}
&\overset{\d}{\approx}
    \underbracket[0.5pt]{\sum_{i=1}^{\smash{\ell}} w_{i} \phi_{i}(\cdot)
    \vphantom{\sum_{j=1}^{\smash{m}} v_{j} k(\cdot, \v{z}_{j})}}_{\t{weight-space prior}}
    + 
    \underbracket[0.5pt]{\sum_{j=1}^{\smash{m}} v_{j} k(\cdot, \v{z}_{j}),}_{\t{function-space update}}
\>
where we have defined $\v{v} = \m{K}_{m,m}^{-1}(\v{u} - \m{\Phi}\v{w})$. The equivalent expression for exact GPs with Gaussian observations is obtained by adding noise $\v{\varepsilon} \sim \c{N}(\v{0}, \sigma^2 \m{I})$ to $\m{\Phi}\v{w}$ and replacing $\m{Z}$, $\v{u}$, and $\m{K}_{m,m}^{-1}$ with $\m{X}$, $\v{y}$, and $(\m{K}_{n,n} + \sigma^{2}\m{I})^{-1}$.

Figure~\ref{fig:matheron_overview} acts as a visual guide for \METHOD{}, showing the progression from prior \eqref{eqn:bayes_linear} to posterior \eqref{eqn:decoupled_sample_path}. Stepping through this example: (i) we draw a function $f$ from an approximate prior, (ii) we construct an update function to account for the residuals $\v{u} - f(\m{Z})$ produced by an independent sample $\v{u} \sim q(\v{u})$, (iii) we add these functions together to obtain a function drawn from an approximate posterior \eqref{eqn:decoupled_sample_path} that we may freely evaluate anywhere in $\c{X}$.

In \eqref{eqn:decoupled_sample_path}, we obtain an efficient approximator by separately discretizing the prior using Fourier basis functions $\phi_i(\cdot)$ and the update using canonical basis functions $k(\cdot, \m{z}_j)$. While other decompositions exist (see Appendix~\ref{apdx:discussion}), this particular decoupling directly capitalizes upon each basis' strengths: the Fourier basis is well-suited for representing the prior \cite{rahimi08} and the canonical basis is well-suited for representing the data \cite{burt19}.

By combining these bases as in \eqref{eqn:decoupled_sample_path}, we therefore inherit the best of both worlds. As in weight-space methods, we may efficiently approximate draws from the prior using an $\ell$-dimensional Bayesian linear model $f(\cdot) = \v{\phi}(\cdot)^{\top}\v{w}$, where weights $\v{w}$ are standard normal (owing to the assumed stationarity of kernel $k$).\footnote{This point was not lost on \textcite{hoffman91}, who similarly approximated stationary priors using spectral methods.} As in function-space methods, we may faithfully represent the data since basis functions $k(\cdot, \v{z}_{j})$ are in one-to-one correspondence with inducing locations $\v{z}_{j} \in \m{Z}$. This retention of statistical propriety is evident on the right-hand side of Figure~\ref{fig:matheron_overview}: despite using half as many basis functions as the weight-space method (see Figure~\ref{fig:comparison_of_posteriors}), \METHOD{}'s statistical properties mirror those of the gold standard.

\begin{figure*}[t]
\center
\includegraphics[width=\textwidth]{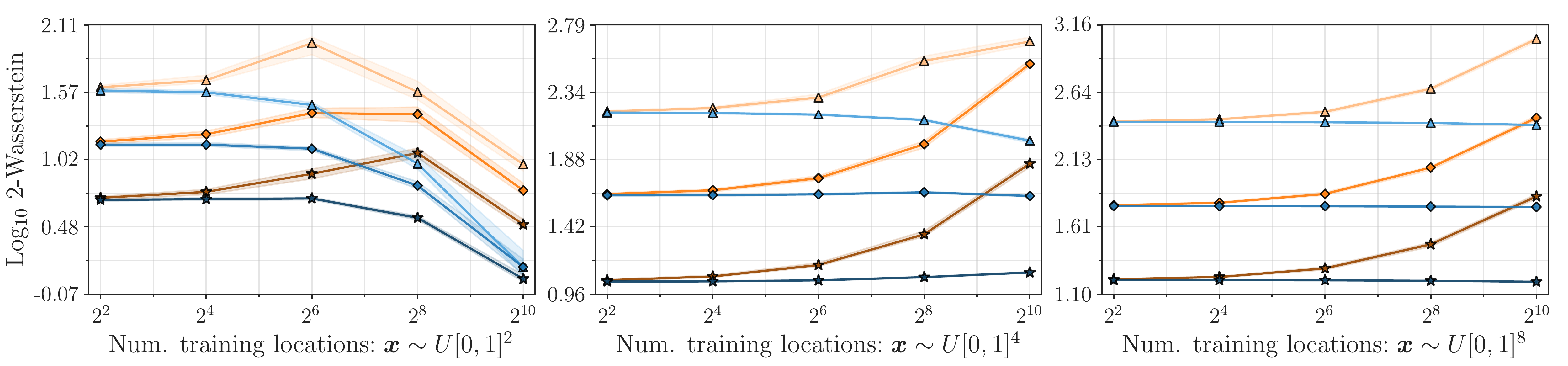}
\caption{Empirical estimates of 2-Wasserstein distances between true posteriors and empirical distributions of $100,000$ samples at 1024 test locations $\m{X}_{*}$ given varying amounts of training data, shown as medians and interquartile ranges~(shaded regions) measured over 64 independent trials.
Weight-space (orange) and decoupled (blue) sampling utilized a total of $\numBasisTotal = m + \ell$~basis functions.
Results using~$\ell \in \{1024, 4096, 16384\}$ initial bases correspond with $\{\text{light}, \text{medium}, \text{dark}\}$ tones and
$\{
    \bigtriangleup, 
    \hbox{\scalebox{1.5}{$\diamond$}},
    \hbox{{\faStarO}}
\}$ 
markers. See Appendix~\ref{apdx:additional_experiments_2wass} for extended results and comparison with LOVE \cite{pleiss2018constant}.
}
\label{fig:wasserstein2_test}
\end{figure*}

Expanding upon these properties, we note the following intuitive behaviors. The update function's role of "correcting" for residuals $\v{u} - f(\m{Z})$ subsumes that of representing the posterior mean: 
replacing the prior draw $f$ with the prior mean $\E[f]$ reduces \eqref{eqn:decoupled_sample_path} to the standard expression for the conditional expectation $\E[f \given \v{u}]$.
Since this task is performed in the canonical basis, the expected value of decoupled sample paths is guaranteed to coincide with that of (sparse) GP's posterior. As a result, \METHOD{} becomes increasingly well-behaved as the number of training (inducing) locations grows and uncertainty decreases. Conversely, we are guaranteed to revert to the prior as we move away from the data, assuming local basis functions $k(\cdot, \v{z})$ (see the center column of Figure~\ref{fig:matheron_overview}).

Decoupled sampling complements these desiderata with function draws' inherent strengths. The immediate implication here is that decoupled sampling scales linearly with respect to the number of test locations $\m{X}_{*}$. A more subtle point is that these functions are \emph{pathwise differentiable} with respect to $\v{x}$---an affordance with significant consequences when seeking to understand Gaussian processes' extrema.

While these insights tell us about \METHOD's qualitative behavior, they do not allow us to make quantitative statements about its purported benefits. To this end, the following section provides a means of objectively comparing different sampling schemes' statistical properties.

\subsection{Error bounds}
\label{sec:error_bound}

Due to its use of an approximate prior, \METHOD introduces an additional source of error at test time.  Anecdotal evidence (see Figure~\ref{fig:matheron_overview}) suggests that this sampling error is often small in comparison to the error introduced by inducing point approximations. Here, we study \METHOD's analytic properties to clarify how quality of the approximate prior impacts that of decoupled function draws. We present the results of this analysis below, and reserve proofs and derivations of associated constants for Appendix~\ref{apdx:error-analysis}. As a convenient shorthand, we refer to the particular decoupled sparse GP approximation introduced in \eqref{eqn:decoupled_sample_path} as \dsgp.

Gaussian processes are often compared via a suitable notion of similarity on the space of probability distributions \cite{gibbs02}.
We focus on the 2-Wasserstein distance between GPs \cite{mallasto2017learning}. By their Kantorovich dual formulations \cite{peyre2019computational}, Wasserstein distances upper bound the error for integrating (Lipschitz) continuous functions with respect to approximating distributions, making them a natural performance metric in Monte Carlo settings.
Moreover, 2-Wasserstein distances between exact GPs and finite-dimensional approximations thereof are finite and thus facilitate meaningful performance comparisons.
For \dsgp{}, we may bound~these~as~follows.

\begin{proposition}
\label{prop:wasserstein_bound}
Assume that $\c{X} \subseteq \mathbb{R}^d$ is compact and that stationary kernel $k$ is sufficiently regular for $f \sim \c{GP}(0, k)$ to be almost surely continuous. Let $f \given \v{y}$ be the posterior of $f$, $f^{(s)}$ that of a sparse GP, and $f^{(d)}$ that of a corresponding \dsgp{} defined via an approximate prior $f^{(w)}$. Then we have 
\vspace*{-4ex}
\[
\begin{aligned}
\label{eqn:wasserstein_bound}
&W_{2,L^2(\c{X})} \del[1]{f^{(d)}, f \given \v{y}}
\\
 &
    \leq 
    \smash{
    \underbracket[0.5pt]{W_{2,L^2(\c{X})}\del[1]{f^{(s)}, f \given \v{y}}}_{\t{error in the (sparse) posterior}}
    + 
    \underbracket[0.5pt]{C_1 W_{2, C(\c{X})}\del[1]{f^{(w)}, f}}_{\t{error in the prior}}
    ,
    }
\end{aligned}
\]
\strut
\\[-0.5ex]
where $W_{2,L^2(\c{X})}$ and $W_{2,C(\c{X})}$ are the 2-Wasserstein distances over $L^2(\c{X})$ and the space of continuous functions $C(\c{X})$ equipped with the supremum norm, respectively.
\end{proposition}
\vspace*{-2.5ex}
\begin{proof}
Appendix \ref{apdx:error-analysis}.
\end{proof}
\vspace*{-1.5ex}

This bound tells us that the error exhibited by \dsgp{} sample paths cleanly separates into independent terms associated with the sparse GP and the approximate prior. In particular, the way in which error in the prior carries over to the posterior is controlled by the inducing locations $\m{Z}$, which $C_{1}$ depends on, but not by the inducing distribution $q(\v{u})$.

\begin{figure*}
\center
\includegraphics[width=\textwidth]{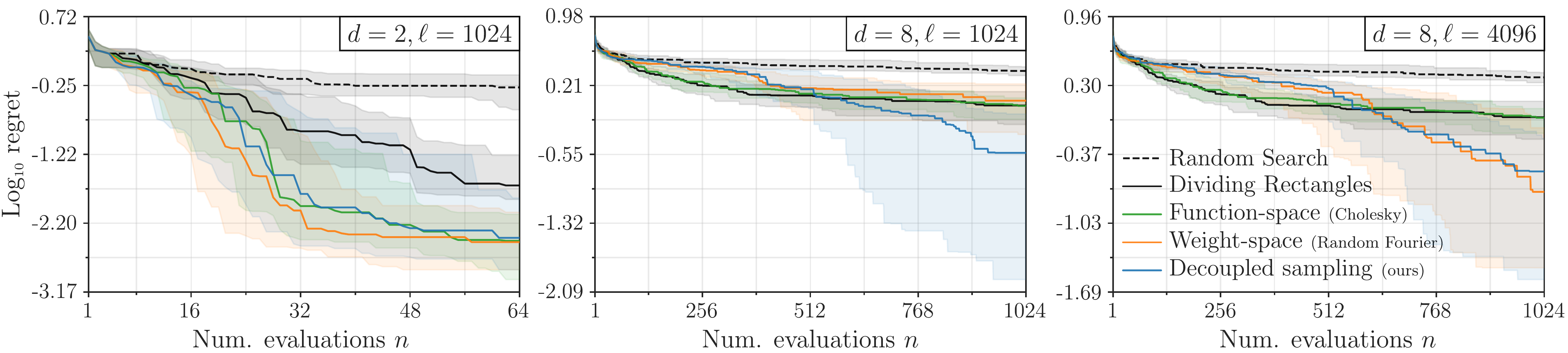}
\caption{Median performances and interquartile ranges of parallel Thompson Sampling (TS) and popular baselines when optimizing $d$-dimensional functions drawn from GP priors. 
Function-space TS delivers competitive performance for $d=2$, but is held back by its inability to efficiently utilize gradient information to combat the curse of dimensionality. \rff-based TS avoids this issue but requires $b \gg m$ basis functions to perform well. TS with decoupling sampling matches or outperforms competing approaches in all observed cases. See Appendix~\ref{apdx:additional_experiments_ts} for additional results and runtime distributions.}
\label{fig:thompson_sampling_results}
\end{figure*}

We continue this analysis by studying \dsgp{}'s moments. Since a \dsgp{}'s mean is guaranteed to coincide with that of a sparse GP, we focus on the error it introduces into the posterior covariance. When using \rff{} to approximate the prior, this error will depend on the $\ell$-dimensional basis $\v{\phi}$ given by parameters $\tau \sim U(0, 2\pi)$ and $\v{\theta} \sim s(\v{\theta})$, where $s(\cdot)$ denotes the (normalized) spectral density of $k$. We therefore bound the expectation of this error.

\vspace*{0.25ex}

\begin{proposition}
\label{prop:covariance_error}
Continuing from Proposition~\ref{prop:wasserstein_bound}, let
$k^{(f \given \v{y})}$, $k^{(w)}$, $k^{(s)}$, $k^{(d)}$ respectively denote the covariance functions of processes $f \given \v{y}$, $f^{(w)}$, $f^{(s)}$, $f^{(d)}$. Denoting the supremum norm over continuous functions by $\norm{\cdot}_{C(\c{X}^2)}$, it follows that
\[
\begin{aligned}
&
    \mathbb{E}_{\v{\phi}}
    \norm[1]{k^{(d)} - k^{(f \given \v{y})}}_{C(\c{X}^2)}
\\
& \quad\quad\leq 
    \norm[1]{k^{(s)} - k^{(f \given \v{y})}}_{C(\c{X}^2)} + \frac{C_2 C_3}{\sqrt{\ell}},
\end{aligned}
\]
where the constants $C_2$ and $C_{3}$ are given by \textcite{sutherland15} and in Appendix~\ref{apdx:error-analysis}, respectively.
\end{proposition}
\vspace*{-2.5ex}
\begin{proof}
Appendix \ref{apdx:error-analysis}.
\end{proof}

Much like \dsgp{}s themselves, the error in the posterior covariance separates into terms associated with the covariance of the sparse GP $k^{(s)}$ and approximate prior $k^{(w)}$. This latter source of error represents discrepancies introduced by using \rff{} to approximate the prior and decays at a \emph{dimension-free rate} as the number of basis functions $\ell$ increases. Intuitively, this behavior reflects \rff{}'s nature as a Monte Carlo estimator of the true covariance. In practice, the number of training points $n$ typically grows faster than the dimensionality $d$. Hence, purely \rff-based GP posteriors struggle to capitalize upon this property due to variance starvation. Since \dsgp{} does not exhibit this pathology, it fully benefits from this dimension-free rate of convergence.

\section{Experiments}
\label{sec:experiments}

We investigate  \METHOD's behavior in a series of sample tests accompanied by two practical applications, Thompson sampling and dynamical system simulation. Each of these experiments highlights different properties of decoupled sample paths: uncertainty calibration, reliability and differentiability, and computational savings.\footnote{Code: \url{https://github.com/j-wilson/GPflowSampling}
}

\paragraph{Uncertainty calibration with the 2-Wasserstein distance.}
To better understand how the bounds presented in Section~\ref{sec:error_bound} manifest in the real world, we put the various sampling schemes through numerical experiments that empirically estimated the 2-Wasserstein distance bounded by \eqref{eqn:wasserstein_bound}. 
These tests allow us to see how this distance is affected by factors, such as the number of training points, whose effects are difficult to directly analyze.
In each trial, we measured the distance between the true posterior and empirical distributions of samples generated using the various strategies introduced in the paper. To eliminate confounding variables, experiments were run using exact GPs with known hyperparameters (see Appendix~\ref{apdx:additional_experiments} for details).

Our investigation focuses on each method's behavior as the number of inducing locations $m$ (equivalently, the number of training points $n$) increases relative to the number of basis functions employed. For fair comparison, the total number of basis functions $\numBasisTotal = m + \ell$ utilized by weight-space and decoupled samplers was held equal, where $\ell$ denotes an initial allocation. For decoupled sampling, $\ell$ specifies the number of Fourier features used to approximate the prior.

Figure \ref{fig:wasserstein2_test} shows that weight-space sampling tends to deteriorate as $m$ increases relative to $b$. Variance starvation causes sample paths' extrapolatory behavior to increasingly misrepresent the posterior. This issue is exacerbated as dimensionality $d$ rises, since we can expect the (randomly chosen) test locations $\m{X}_{*}$ to lie further and further away from the data.

In contrast, \METHOD{} retains its performance, and may even improve. This reflects the fact that the data is represented in an efficient basis that grows alongside it. For sparse GPs with $m \ge n$ (which includes exact GPs as a special case), we may always represent the data exactly: usually, however, $m \ll n$ inducing locations (i.e., kernel basis functions) suffice \cite{burt19}. Since we expect posteriors to contract as training sets expand, the functions drawn from these posteriors behave increasingly similar to their means. Since decoupled sample paths are guaranteed to exhibit the correct means, their statistical properties may improve. This process occurs more slowly in higher-dimensional cases. However, since away from data these function draws revert to the approximate prior, they exhibit constant error when extrapolating---the approximation error of said prior.

\begin{figure*}[t]
\center
\includegraphics[width=\textwidth]{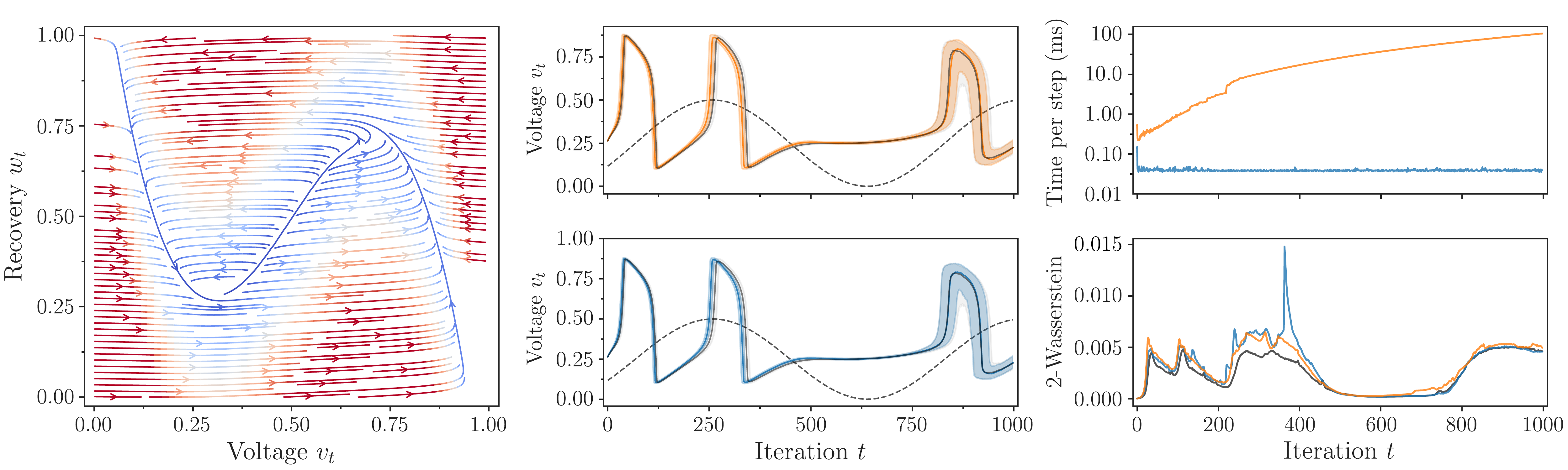}
\caption{
Sparse GP-based simulation of a FitzHugh-Nagumo model neuron subject to evolution noise $\v{\varepsilon}_{t} \sim \c{N}(\v{0}, 10^{-2}\m{I})$ and current injection $I(t) \in \mathbb{R}$. \emph{Left:} True drift function $f$ given a fixed current $I(t) = 0.5$. \emph{Middle:} Medians and interquartile ranges of 1000 voltage traces generated in response to a sinusoidal control signal (dashed black)
using iterative (orange) and decoupled (blue) sampling are compared with those of ground truth simulations (gray). \emph{Upper right:} Runtime comparison of iterative and decoupled sampling: the former scales cubically, while the latter runs in linear time. \emph{Lower right:} 
2-Wasserstein distances between state distributions at times $t$ are approximated using the Sinkhorn algorithm \cite{cuturi2013sinkhorn}. The noise-floor (gray) was established using additional ground truth simulations.
}
\label{fig:fitzhugh_nagumo}
\end{figure*}

\paragraph{Thompson Sampling with reliable, differentiable draws.} Thompson Sampling (TS)
is a classic strategy for decision-making in the face of uncertainty, whereby a choice $\v{x} \in \c{X}$ is selected according to its estimated probability of being optimal \cite{thompson33}.
When used as a vehicle for GP-based optimization, TS evaluates a pathwise minimizer
\[
    \v{x}_{n+1} \in \argmin_{\v{x} \in \c{X}} (f \given \v{y})(\v{x})
\]
of a function drawn $f \given \v{y}$ from the posterior. Upon finding this minimizer, $\v{x}_{n+1}$ is evaluated to obtain $y_{n+1}$, the pair $(\v{x}_{n+1},y_{n+1})$ is added to the training set, and the process repeats. In practice, this algorithm is (embarrassingly) parallelized by independently drawing $\kappa > 1$ functions and evaluating a minimizer of each one \cite{hernandez2017parallel,kandasamy2018parallelised}.

We compare the performance of parallel TS equipped with the various sampling schemes discussed in Section~\ref{sec:methods}, along with two common baselines. To help eliminate confounding variables, experiments were run using functions drawn from known GP priors with fixed measurement noise $y_i \sim \c{N}(f_i, 10^{-3})$. 
Across trials, we varied both the dimensionality $d$ of search spaces $\c{X} = [0, 1]^{d}$ and the number of initial basis functions $\ell$. We set $\kappa = d$, but this choice was not found to greatly influence results. The total number of basis functions allocated to weight-space and decoupled samplers was again matched, so that $\numBasisTotal = m + \ell$.

Figure~\ref{fig:thompson_sampling_results} shows that different methods of sampling from GP posteriors dramatically influence achieved performance.
While all methods suffered from the curse of dimensionality, TS in function-space deteriorates most aggressively, owing to its inability to efficiently exploit gradient information and to the prohibitive cost for generating large sample vectors $\v{f}_{\test} \given \v{y}$. Weight-space TS resolves both of these issues and, therefore, performs competitively\textemdash so long as $\numBasisTotal \gg m$, in which case it accurately approximates the posterior. On the other hand, TS in weight-space collapses due to variance starvation as $m$ increases relative to $b$, often performing worse than simpler alternatives.

\xmakefirstuc{\METHOD{}} avoids these shortcomings. As function draws, decoupled sample paths $(f \given \v{y})(\m{X}_{*})$ boast linear time complexity $\c{O}(*)$ and can be minimized by pathwise differentiating with respect to $\m{X}_{*}$. Moreover, because the canonical basis is able to efficiently represent the data, these sample paths retain their statistical properties even when $\numBasisTotal$ is comparable to $n$ or, in the case of sparse GPs, when $\numBasisTotal \ll n$.

\newcommand{\state}{\ensuremath{s}}
\newcommand{\ctrl}{\ensuremath{c}}
\textbf{Simulating dynamical systems in linear time.} Model-based simulators are commonly used in cases where real-world data collection proves impractical or impossible. For example, GP surrogates are a key component of state-of-the-art methods for solving the types of continuous control problems seen in robotics \cite{Deisenroth2015, kamthe2017data}.  Without loss of generality, we assume that our goal is to model a time-invariant system whose dynamics are governed by a stochastic differential equation, discretized according to the the Euler-Maruyama integrator
\[
\Delta \v{s}_{t}
    = \v{\state}_{t + 1} - \v{\state}_{t}
    = f(\v{\state}_{t}, \v{\ctrl}_{t}) \Delta t + \sqrt{\Delta t \m{\Sigma}} \v{\varepsilon}_{t},
\label{eqn:euler_maruyama}
\]
where $\v{\state}_{t}$ denotes the state at time $t$, $\v{\ctrl}_{t} \in \c{U} \subseteq \mathbb{R}^{\ctrl}$ a control input, and $\v{\varepsilon}_{t} \sim \c{N}(\v{0}, \m{I})$ a standard normal random vector. 

Having trained a (sparse) GP to represent possible drift functions $f$, we simulate the system's evolution over time by \emph{unrolling}: given a state-control pair $(\v{\state}_{t}, \v{\ctrl}_{t})$, we sample a transition $\Delta \v{\state}_{t}$ according to the GP posterior and step as in \eqref{eqn:euler_maruyama}. Since the resulting trajectory $\m{\uppercase{\state}}_{1:t}$ is determined online, standard approaches to sampling require us to iteratively condition on the preceding sample $f_{t}$ when drawing $f_{t+1} \given \v{f}_{1:t}$. Use of caching and rank-1 downdates help limit associated costs however, the resulting algorithm's time complexity still scales cubically in the number of steps $t$ (see Appendix~\ref{apdx:additional_experiments_dyn}). By virtue of drawing functions, \METHOD{} avoids this machinery and allows us to simulate trajectories in linear time $\c{O}(t)$. 

To better understand the practical ramifications of unrolling with decoupled samples, we used a sparse GP to simulate the dynamics of a well-known model of a biological neuron \cite{fitzhugh1961impulses, nagumo1962active}. Results are shown in Figure~\ref{fig:fitzhugh_nagumo}. For both sampling schemes, simulated trajectories accurately characterizes the ways in which the system may respond to a given control signal. Their respective costs, however, vary dramatically: simulations that required 10 hours using the iterative approach, owing to cubic costs, ran in 20 seconds using decoupled sampling while achieving similar accuracy.

\section{Conclusion}
\label{sec:conclusion}

Decomposing Gaussian processes is a general strategy for constructing efficient approximation schemes. We have focused on a particular case, where a posterior is seen as the sum of a prior and an update, and shown how this decoupling can be exploited to efficiently draw functions from said posterior. Even within this choice of decomposition however, optimal treatment of these components will ultimately depend upon the nature of the task at hand. For example, when working with structured covariance matrices, it is sometimes possible to efficiently generate draws from the prior without introducing approximation error \cite{dietrich97, wilson2015kernel}. These alternatives can then be combined with ideas discussed in previous sections to achieve the desired balance of speed versus accuracy.

Owing to the generality of our assumptions and simplicity of our proposals, \METHOD{} can be used as a plug-in extension to existing sample-based algorithms driven by (sparse) GPs. Separately representing the prior and the data with bases better suited for sampling allows us to obtain the "best of both worlds" by bringing together previous methods' strengths. The result of this union, \METHOD{}, draws functions from GPs that may be evaluated in linear time without fear of misrepresenting their posteriors.

\section*{Acknowledgments}
The authors would like to express their gratitude to Prof. Mikhail Lifshits, without whom our collaboration would never have started. This research was partially supported by "Native towns", a social investment program of PJSC "Gazprom Neft" and by the Ministry of Science and Higher Education of the Russian Federation, agreements N\textsuperscript{\underline{o}} 075-15-2019-1619 and N\textsuperscript{\underline{o}} 075-15-2019-1620. The support of the EPSRC Centre for Doctoral Training in High Performance Embedded and Distributed Systems (reference EP/L016796/1) is gratefully acknowledged.

\printbibliography

\onecolumn
\appendix

\section{Alternative decompositions}
\label{apdx:discussion}

As mentioned in the Section~\ref{sec:decoupled_sampling}, the proposed representation of the GP posteriors---as the sum of a weight-space prior and a function-space update---is one of many possible choices. Here, we briefly reflect on two such alternatives. 

To begin with, we may directly represent sparse GP posteriors in weight-space via a Bayesian linear model $f(\cdot) = \v{\phi}(\cdot)^{\top}\v{w}$. To this end, we may rewrite \eqref{eqn:matheron_gp_rff} for a given draw $\v{u} \sim q(\v{u})$ as
\[
\v{w} \given \v{u} 
    \overset{\d}{=} 
        \v{w} + \m{\Phi}^{\top}
        (\m{\Phi}\m{\Phi}^{\top})^{-1}
        (\v{u} - \m{\Phi}\v{w}),
    \label{eqn:matheron_gp_rff_sparse}
\]
where $\m{\Phi} = \v{\phi}(\m{Z})$ now denotes an $m \times \ell$ feature matrix. Prima facie, this appears to resolve many of the problems discussed earlier in the text: inducing distribution $q(\v{u})$ relays information about $\v{y}$ and the Bayesian linear model needs only explain for the function's behavior at $m \ll n$ locations. In practice, \eqref{eqn:matheron_gp_rff_sparse} does more harm than good however, since $f$ must now exactly pass through $\v{u}$ due to a lack of measurement noise $\sigma^{2}$.

Alternatively, we may think to employ an \emph{orthogonal decomposition} $f(\cdot) = f_{\parallel}(\cdot) + f_{\bot}(\cdot)$ \cite{salimbeni18,shi19}. Here, we interpret "orthogonality" in the statistical sense of independent random variables \cite{rodgers1984linearly}. For Gaussian random variables, this distinction amounts to satisfying the definition $\Cov(f_{\parallel}, f_{\bot}) = 0$. In the case of sparse GPs, $f_{\parallel}$ is typically represented in terms of canonical basis functions $k(\cdot, \m{Z})$ such that $(f_{\parallel} \given \v{u})(\cdot)$ denotes the posterior mean function given $q(\v{u})$. Consequently, $f_{\bot}$ denotes the process residuals $(f_{\bot} \given \v{u})(\cdot) = (f \given \v{u})(\cdot) - (f_{\parallel} \given \v{u})(\cdot)$. By construction however, $f_{\bot}$ is independent of $f_{\parallel}$ and, hence, of particular values $\v{u}$. Moreover, since $(f \given \v{u})(\m{Z}) = (f_{\parallel} \given \v{u})(\m{Z}) = \v{u}$, it follows that $f_{\bot}(\m{Z}) = (f_{\bot} \given \v{u})(\m{Z}) = \m{0}$. 

Generating draws from this type of decomposition is made difficult by orthogonal component $f_{\bot} \given \v{u}$, whose covariance can readily be shown as
\[
\Cov(f_{\bot}, f_{\bot})
    = k(\cdot, \cdot) - k(\cdot, \m{Z})\m{K}_{m, m}^{-1} k(\m{Z}, \cdot).
\]
Sampling schemes based on random Fourier feature approximations of $f_{\bot}$ are nearly identical to \eqref{eqn:matheron_gp_rff_sparse}: all that has changed is that the Bayesian linear model must now pass exactly through zero, rather than $\v{u}$, at each of the $m$ inducing locations. This approach to sampling therefore inherits the issues outlined above.

\section{Error analysis}
\label{apdx:error-analysis}

\begin{definition}[Preliminaries]
Consider a Gaussian process $f$ defined on $\R^d$ and restricted to a compact subset $\c{X} \subseteq \R^d$. 
Let $\v{y} \in \R^n$.
Assume a Gaussian likelihood $y_i \sim \c{N}(f(x_i),\sigma^2)$, with $\sigma^2 \geq 0$.
Let $f^{(w)}$ be a weight-space prior approximation.
Let $f \given \v{y}$ be the true posterior, let $f^{(s)}$ be an inducing point approximate posterior, and let $f^{(d)}$ be the decoupled posterior approximation.
Let $k, k^{(w)}, k^{(f\given\v{y})}, k^{(s)}, k^{(d)}$ be their respective kernels.
\end{definition}

\begin{proposition}
We have that
\[
W_{2,L^2(\c{X})}\del{f^{(d)}, f \given \v{y}} \leq W_{2,L^2(\c{X})}\del{f^{(s)}, f \given \v{y}} + C_1\,W_{2,L^\infty(\c{X})}\del{f^{(w)}, f}
\]
where $C_1 = \sqrt{2 \f{diam} (\c{X})^d \left( 1 + \norm[1]{k}_{C(\c{X}^2)}^2 \norm{\m{K}_{mm}^{-1}}_{L(\ell^\infty; \ell^1)}^2 \right)}$, $W_{2,L^2(\c{X})}$ and $W_{2,C(\c{X})}$ are the 2-Wasserstein distances over $L^2(\c{X})$ and the space of continuous functions $C(\c{X})$ equipped with the supremum norm, respectively, and $\norm{\cdot}_{L(\ell^\infty; \ell^1)}$ is the corresponding operator norm of a matrix.
\end{proposition}

\begin{proof}
By the triangle inequality, we have
\[
W_{2,L^2(\c{X})}\del{f^{(d)}, f \given \v{y}} \leq W_{2,L^2(\c{X})}\del{f^{(d)}, f^{(s)}} + W_{2,L^2(\c{X})}\del{f^{(s)}, f \given \v{y}}
.
\]
We proceed bound the first term pathwise.
For arbitrary $x\in M$, write
\<
\label{eqn:error-analysis-line-1}
\abs{f^{(d)}(x) - f^{(s)}(x)}^2 &\leq 2 \del{\abs{f^{(w)}(x) - f(x)}^2 + \abs{\m{K}_{xm}\m{K}_{mm}^{-1}(f^{(w)}(\v{z}) - f(\v{z}))}^2}
\\
\label{eqn:error-analysis-line-2}
&\leq 2 \del{\norm{f^{(w)} - f}_{L^\infty(\c{X})}^2 + \norm{\m{K}_{xm}\m{K}_{mm}^{-1}}_{\ell^1}^2 \norm{f^{(w)}(\v{z}) - f(\v{z})}_{\ell^\infty}^2}
\\
\label{eqn:error-analysis-line-3}
&\leq 2 \del{\norm{f^{(w)} - f}_{L^\infty(\c{X})}^2 + \norm{\m{K}_{xm}}_{\ell^\infty}^2\norm{\m{K}_{mm}^{-1}}_{L(\ell^\infty; \ell^1)}^2 \norm{f^{(w)} - f}_{L^\infty(\c{X})}^2}
\\
\label{eqn:error-analysis-line-4}
&\leq 2 \del{1 + \norm[1]{k}_{C(\c{X}^2)}^2 \norm{\m{K}_{mm}^{-1}}_{L(\ell^\infty; \ell^1)}^2} \norm{f^{(w)} - f}_{L^\infty(\c{X})}^2
\\
\label{eqn:error-analysis-line-5}
&= 2 \del{1 + \norm[1]{k}_{C(\c{X}^2)}^2 \norm{\m{K}_{mm}^{-1}}_{L(\ell^\infty; \ell^1)}^2} \norm{f^{(w)} - f}_{C(\c{X})}^2
\>
where in \eqref{eqn:error-analysis-line-1} we have used Matheron's rule, in \eqref{eqn:error-analysis-line-2} we have used H\"{o}lder's inequality with $p=1, q=\infty$, in \eqref{eqn:error-analysis-line-3} we have used the definition of an operator norm, and in \eqref{eqn:error-analysis-line-5} we have used that given sample paths are continuous so $\norm{\cdot}_{L^\infty(\c{X})}$ can be replaced with $\norm{\cdot}_{C(\c{X})}$.
We now lift this to a bound on the Wasserstein distance by integrating both sides.
With $\gamma \in \c{C}$ denoting couplings between $\c{GP}(0, k)$ and $\c{GP}(0, k^{(w)})$, write
\<
W_{2,L^2(\c{X})}^2(f^{(d)}, f^{(s)}) &\leq \inf_{\gamma\in\c{C}} \int \norm{f^{(d)} - f^{(s)}}_{L^2(\c{X})}^2 \d\gamma
\\
&\leq  C |\c{X}| \inf_{\gamma\in\c{C}} \int \norm{f^{(w)} - f}_{C(\c{X})}^2 \d\gamma
\\
&=  C \f{diam}(\c{X})^d\, W_{2,C(\c{X})}^2(f^{(w)}, f)
\>
where $C$ is the constant above.
Finally, note that $f$ is sample-continuous, and $C(\c{X})$ is a separable metric space, so $W_{2,C(\c{X})}$ is a proper metric. 
The claim follows.
\end{proof}

\begin{proposition}
Assume $k$ is stationary continuous covariance defined on $\mathbb{R}^d \times \mathbb{R}^d$, $\c{X} \subseteq \mathbb{R}^d$ is compact.
We have that
\[
\E_{\substack{\v\omega \sim \rho\\\upsilon \sim U}} \norm{k^{(d)} - k^{(f \given \v{y})}}_{C(\c{X}^2)} \leq \norm{k^{(s)} - k^{(f \given \v{y})}}_{C(\c{X}^2)} + \frac{C_2 C_3}{\sqrt{\ell}}
\]
where $\norm{\cdot}_{C(\c{X}^2)}$ is the supremum norm over continuous functions, $C_2$ is the constant given by \textcite{sutherland15}, which depends only on the Lipschitz constant of $k$, the rate of decay of the spectral density $\rho$, the dimension $d$, and the diameter of the domain $\c{X}$, and $C_3 = m \sbr{1 + \norm{\m{K}^{-1}_{m,m}}_{C(\c{X}^2)} \norm{k}_{C(\c{X}^2)}}^2$.
\end{proposition}

\begin{proof}
By the triangle inequality, we have
\[
\E_{\substack{\v\omega \sim \rho\\\upsilon \sim U}} \norm{k^{(d)} - k^{f \given \v{y}}}_{C(\c{X}^2)} \leq \E_{\substack{\v\omega \sim \rho\\\upsilon \sim U}} \norm{k^{(d)} - k^{(s)}}_{C(\c{X}^2)} + \norm{k^{(s)} - k^{f \given \v{y}}}_{C(\c{X}^2)}
\]
where we have used that the latter term does not depend on $\v\omega$.
We proceed to bound the inner portion of the first term.
Define the bounded linear operator $M_k: C(\c{X}\times\c{X})\to C(\c{X}\times\c{X})$ by the expression
\[
(M_k c)(x,x') = c(x,x') - \m{C}_{x,m}\m{K}^{-1}_{m,m}\m{K}_{m,x'} - \m{K}_{x,m}\m{K}^{-1}_{m,m}\m{C}_{m,x'} + \m{K}_{x,m}\m{K}^{-1}_{m,m}\m{C}_{m,m}\m{K}^{-1}_{m,m}\m{K}_{m,x'}
.
\]
Let $\m\Sigma = \Cov(\v{u})$. By explicit calculation, we have
\[
k^{(d)}(x,x') = (M_k k^{(w)})(x,x') + \m{K}_{x,m}\m{K}^{-1}_{m,m}\m\Sigma\m{K}^{-1}_{m,m}\m{K}_{m,x'}
\]
and we also have
\[
k^{(s)}(x,x') = k^{(f\given\v{y})}(x,x') + \m{K}_{x,m}\m{K}^{-1}_{m,m}\m\Sigma\m{K}^{-1}_{m,m}\m{K}_{m,x'}
\]
hence
\[
\norm{k^{(d)} - k^{(s)}}_{C(\c{X}^2)} = \norm[1]{M_k k^{(w)} - k^{(f\given\v{y})}}_{C(\c{X}^2)} = \norm[1]{M_k k^{(w)} - M_k k}_{C(\c{X}^2)} \leq \norm[1]{M_k}_{L(C;C)} \norm[1]{k^{(w)} - k}_{C(\c{X}^2)}
.
\]
We proceed to bound the operator norm  $\norm[1]{M_k}_{L(C;C)}$.
Write
\<
\norm{M_k c}_{C(\c{X}^2)} &\leq \norm{c}_{C(\c{X}^2)} + \norm{\m{C}_{\cdot,m}\m{K}^{-1}_{m,m}\m{K}_{m,\cdot}}_{C(\c{X}^2)} + \norm{\m{K}_{\cdot,m}\m{K}^{-1}_{m,m}\m{C}_{m,\cdot}}_{C(\c{X}^2)}
\\
&\quad+ \norm{\m{K}_{\cdot,m}\m{K}^{-1}_{m,m}\m{C}_{m,m}\m{K}^{-1}_{m,m}\m{K}_{m,\cdot}}_{C(\c{X}^2)}
.
\>
Now, note that
\<
\norm{\m{C}_{\cdot,m}\m{K}^{-1}_{m,m}\m{K}_{m,\cdot}}_{C(\c{X}^2)} &= \sup_{x,x' \in \c{X}} \sbr{ \m{C}_{x,m}\m{K}^{-1}_{m,m}\m{K}_{m,x'} } 
\\
&\leq \sup_{x,x' \in \c{X}} \sbr{ \norm{\m{C}_{x,m}}_{\ell^\infty} \norm{\m{K}^{-1}_{m,m}}_{L(\ell^\infty; \ell^1)} \norm{\m{K}_{m,x'}}_{\ell^\infty} } 
\\
&\leq \norm{c}_{C(\c{X}^2)} \norm{\m{K}^{-1}_{m,m}}_{L(\ell^\infty; \ell^1)} \norm{k}_{C(\c{X}^2)}
\>
by H\"{o}lder's inequality with $p=1$ and $q=\infty$, and then by the definition of the operator norm $\norm{\cdot}_{L(\ell^\infty; \ell^1)}$. Similarly
\[
\norm{\m{K}_{\cdot,m}\m{K}^{-1}_{m,m}\m{C}_{m,m}\m{K}^{-1}_{m,m}\m{K}_{m,\cdot}}_{C(\c{X}^2)} \leq m \norm{c}_{C(\c{X}^2)} \norm{\m{K}^{-1}_{m,m}}_{L(\ell^\infty; \ell^1)}^2 \norm{k}_{C(\c{X}^2)}^2
\]
hence
\<
\norm{M_k c}_{C(\c{X}^2)} &\leq \norm{c}_{C(\c{X}^2)} + 2\norm{c}_{C(\c{X}^2)}\norm{\m{K}^{-1}_{m,m}}_{L(\ell^\infty; \ell^1)} \norm{k}_{C(\c{X}^2)} + m \norm{c}_{C(\c{X}^2)}\norm{\m{K}^{-1}_{m,m}}_{L(\ell^\infty; \ell^1)}^2 \norm{k}_{C(\c{X}^2)}^2 
\\
&\leq \norm{c}_{C(\c{X}^2)} \left( m \sbr{1 + \norm{\m{K}^{-1}_{m,m}}_{L(\ell^\infty; \ell^1)} \norm{k}_{C(\c{X}^2)}}^2 \right)
\>
and therefore
\[
\norm{M_k}_{L(C;C)} = \sup_{c \neq 0} \frac{\norm{M_k c}_{C(\c{X}^2)}}{\norm{c}_{C(\c{X}^2)}} \leq m \sbr{1 + \norm{\m{K}^{-1}_{m,m}}_{L(\ell^\infty; \ell^1)} \norm{k}_{C(\c{X}^2)}}^2
.
\]
Note that this term is independent of $\v\omega$, and hence constant with respect to the expectation.
Finally, \textcite{sutherland15} have shown that there exists a constant $C_2$ such that.
\[
\E_{\substack{\v\omega \sim \rho\\\upsilon \sim U}} \norm[1]{k^{(w)} - k}_{C(\c{X}^2)} \leq \frac{C_2}{\sqrt{\ell}}
.
\]
Putting together all the inequalities gives the result.
\end{proof}

\section{Additional experiments}
\label{apdx:additional_experiments}

This appendix provides additional details regarding experiments discussed in Section~\ref{sec:experiments}. All experiments (and figures) were run using zero-mean GP priors with Mat\'{e}rn-$\nicefrac{5}{2}$ kernels. For dynamical systems experiments, hyperparameters were learned (MLE type-2). In all other cases, hyperparameters were assumed to be known and specified as: lengthscales $l = \sqrt{\nicefrac{d}{100}}$, measurement noise variance $\sigma^{2} = 10^{-3}$, and kernel amplitude $\alpha = 1$.

\subsection{2-Wasserstein sample tests} 
\label{apdx:additional_experiments_2wass}
In each trial, a set of training locations $\m{X} \sim U[0, 1]^{n \times d}$ was randomly generated and corresponding observations $\v{y} \sim \c{N}(\v{0}, \m{K}_{n,n} + \sigma^{2}\m{I})$ were subsequently drawn from the prior. Similarly, test sets $\m{X}_{*} \sim U[0, 1]^{* \times d}$ were sampled uniformly at random. For each sampling schemes, $100,000$ draws $\v{f}_{*} \given \v{y}$ were then used to form an unbiased estimate $(\tilde{\v{m}}_{* \given n}, \widetilde{\m{K}}_{*,* \given n})$ to the true posterior moments $(\v{m}_{* \given n}, \m{K}_{*,* \given n})$. Given both sets of moments, 2-Wasserstein distances were computed as
\begin{align}
\begin{split}
    &W_{2,\ell^{2}_*}
    \left(
        \c{N}(\v{m}_{*\given n}, \m{K}_{*,* \given n}), 
        \c{N}(\tilde{\v{m}}_{*\given n}, \widetilde{\m{K}}_{*,* \given n})
    \right)^2
    =\\
    &\qquad \qquad
    \norm{\v{m}_{*\given n} - \widetilde{\v{m}}_{*\given n}}^2
    +
    \tr\left(
        \m{K}_{*,* \given n}
        + \widetilde{\m{K}}_{*,* \given n}
        -2 \left(
            \m{K}_{*,* \given n}^{\nicefrac{1}{2}}
            \widetilde{\m{K}}_{*,* \given n}
            \m{K}_{*,* \given n}^{\nicefrac{1}{2}}
        \right)^{\nicefrac{1}{2}}
    \right),
\end{split}
\end{align}
where $\m{K}_{*,* \given n}^{\nicefrac{1}{2}}$ denotes the symmetric matrix square root, and $W_{2,\ell^{2}_*}$ denotes the 2-Wasserstein distance between probability measures over $*$-dimensional vectors equipped with Euclidean distance.

As an additional baseline, we compared decoupled sampling with a LanczOs Variance Estimates (LOVE) based alternative \cite{pleiss2018constant}. The LOVE approach to sampling from GP posteriors exploits structured covariance matrices in conjunction with fast (approximate) solvers to achieve linear time complexity with respect to number of test locations. For example, when inducing locations $\m{Z}$ are defined to be a regularly spaced grid, the prior covariance $\m{K}_{m,m} = k(\m{Z}, \m{Z})$ can be expressed as the Kronecker product of Toeplitz matrices---a property that can be used to dramatically expedite much of the related linear algebra \cite{zimmerman1989computationally, saatcci2012scalable, wilson2015kernel}.

Here, we are interested in comparing the performance of sampling schemes themselves and not that of approximate GPs. As before, we will therefore sample from exact GPs with known hyperparameters. As an additional caveat however, we now define training locations as regularly spaced grids, such that LOVE may represent the data exactly. Similarly, we allow LOVE to utilize $n$ conjugate gradient iterations during precomputation.

\begin{figure*}
\center
\includegraphics[width=\textwidth]{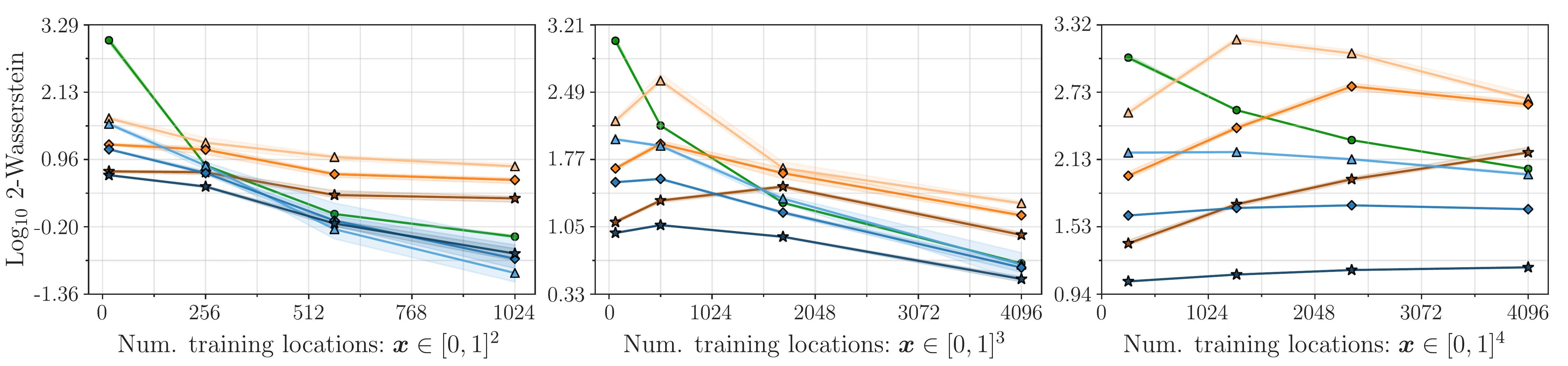}
\caption{Medians and interquartile ranges of empirically estimated 2-Wasserstein distances measured over 32 independent trials consisting of 100,000 samples. LOVE (green) improves as the regularly spaced grids of training locations fill the space. Weight-space (orange) and decoupled (blue) sampling utilized a total of $\numBasisTotal = m + \ell$~basis functions. Results using~$\ell \in \{1024, 4096, 16384\}$ initial bases correspond with $\{\text{light}, \text{medium}, \text{dark}\}$ tones and
$\{
    \bigtriangleup, 
    \hbox{\scalebox{1.5}{$\diamond$}},
    \hbox{{\faStarO}}
\}$ 
markers.
}
\label{fig:wasserstein2_test_extended}
\end{figure*}

Results of these experiments are show in Figure~\ref{fig:wasserstein2_test_extended}. LOVE's performance improves significantly as $m = n$ increases but still lags behind that of decoupled sampling for matching $m$. Several points are immediately worth addressing here. First, kernel interpolation methods such as LOVE offer improved scaling w.r.t. $m$ when compared to na\"ive inducing point methods (even when additional structure is imposed on $\m{Z}$). LOVE can therefore utilize many more inducing locations than traditional sparse GPs in exchange for the imposed structural constraints. Assessing the relative merits of these inducing paradigms is beyond the scope of this work. Second, during sample generation, LOVE exhibits $\c{O}(m + *)$ time complexity, compared to decoupled sampling's $\c{O}(m \times *)$. Third, LOVE samples function values $\v{f}_{*}$ at locations $\m{X}_{*}$ whereas decoupled sampling generates function draws $(f \given \v{u})(\cdot)$, the implications of which were previously explored in Section~\ref{sec:experiments}. Fourth and finally, the techniques and ideas espoused by these frameworks are complementary: just as we may approximate the prior via a collection of Fourier features, we may approximate the update via, e.g., kernel interpolation.

\subsection{Thompson sampling}
\label{apdx:additional_experiments_ts}
As baselines, we compared against Random Search \cite{bergstra2012random} and Dividing Rectangles \cite{jones1993lipschitzian}, the latter of which was run in strictly sequential fashion (i.e., $\kappa = 1$). Minimization tasks were drawn from a known GP prior (see above) and their global minimums were estimated by running gradient descent from a large number of starting locations (for purposes of measuring regret). Here, we discuss algorithmic differences between variants of TS.

For function-space TS, batches were constructed as follows.
\begin{enumerate}[itemsep=0pt,topsep=2pt]
    \item Construct a mesh $\m{X}_{*}$ consisting of $\vert \m{X}_{*} \vert = 10^{6}$ random points.
    \item Draw a vector of independent values $\v{f}_{*} \given \v{y} \sim \c{N}(\v{m}_{* \given n}, \m{K}_{*, * \given n} \odot \m{I})$, where $\odot$ is the element-wise product.
    \item Define an active set $\m{X}_{s} \subseteq \m{X}_{*}$ corresponding to the $s = 2048$ smallest elements of $\v{f}_{*} \given \v{y}$.
    \item Jointly sample $\v{f}_{s} \given \v{y} \sim \c{N}(\v{m}_{s \given n}, \m{K}_{s, s \given n})$.
    \item Select $\v{x}_{i} \in \argmin_{1 \le i \le s} \v{f}_{s} \given \v{y}$ as the $i$-th batch element.
\end{enumerate}
For simplicity, a new mesh $\m{X}_{*}$ was generated at each TS iteration and shared between batch elements, but steps (2-5) we run independently. Weight-space and decoupled TS employed a similar procedure, with minor differences stemming from use of function draws.
\begin{enumerate}[itemsep=0pt,topsep=2pt]
    \item Construct a mesh $\m{X}_{*}$ consisting of $\vert \m{X}_{*} \vert = 250,000$ random points.
    \item Generate a function draw $(f \given \v{y})(\cdot)$.
    \item Define starting locations $\m{X}_{s} \subseteq \m{X}_{*}$ corresponding to the $s = 32$ smallest elements of $(f \given \v{y})(\m{X}_{*})$.
    \item Run multi-start gradient-based optimization: we employed an off-the-shelf version of L-BFGS-B.
    \item Select $\v{x}_{i} \in \argmin_{1 \le i \le s} (f \given \v{y})(\m{X}_{*}^\prime)$ as the $i$-th batch element, where $\m{X}_{s}^\prime$ denotes the optimized locations.
\end{enumerate}
As before, steps (2-5) we run independently. Optimization performance and runtimes are shown below.

\begin{figure*}
\center
\includegraphics[width=\textwidth]{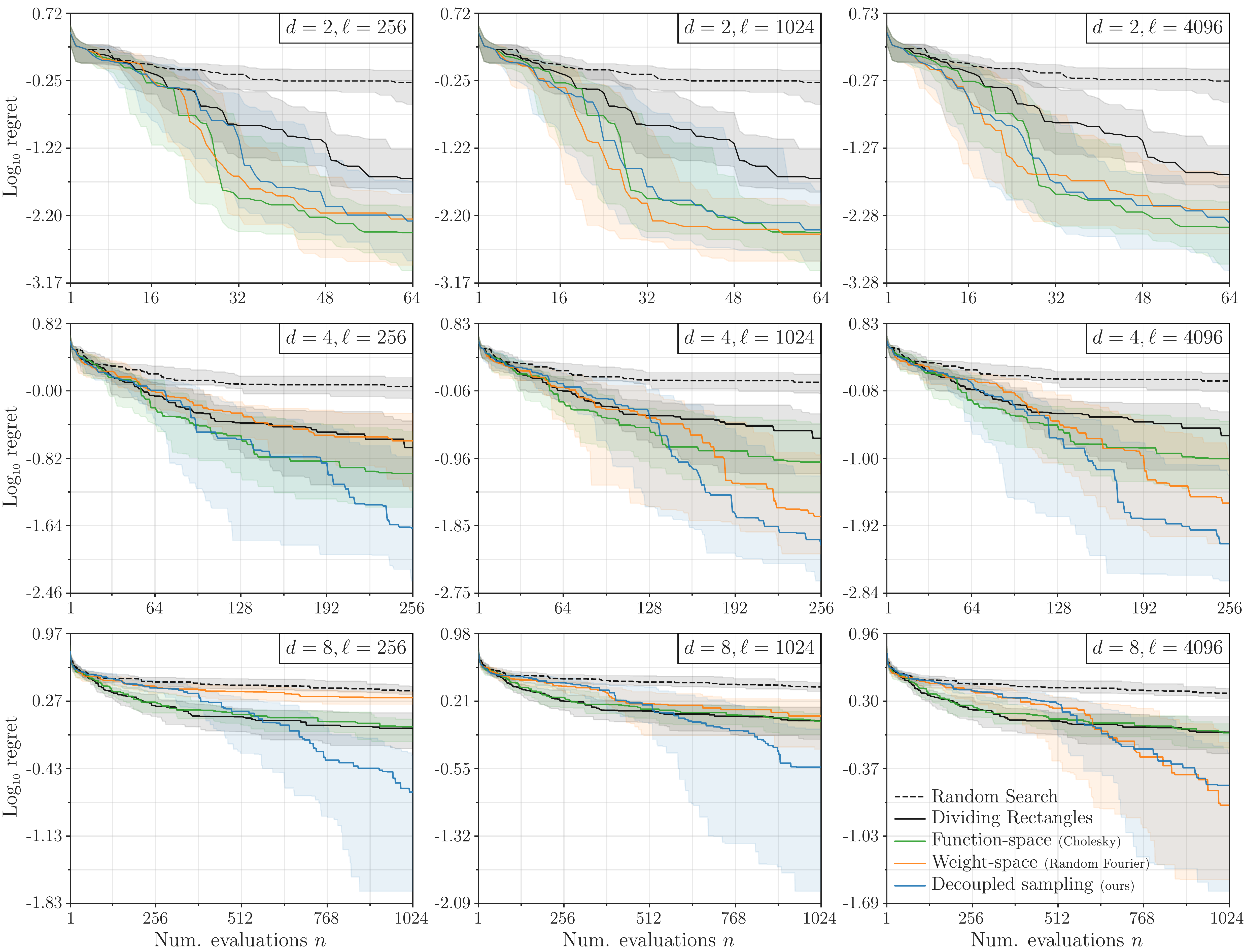}
\caption{Results for parallel Thompson sampling, shown as quartiles over 32 independent runs with matched seeds.}
\label{fig:thompson_sampling_results_full}
\end{figure*}

\begin{figure*}
\center
\includegraphics[width=\textwidth]{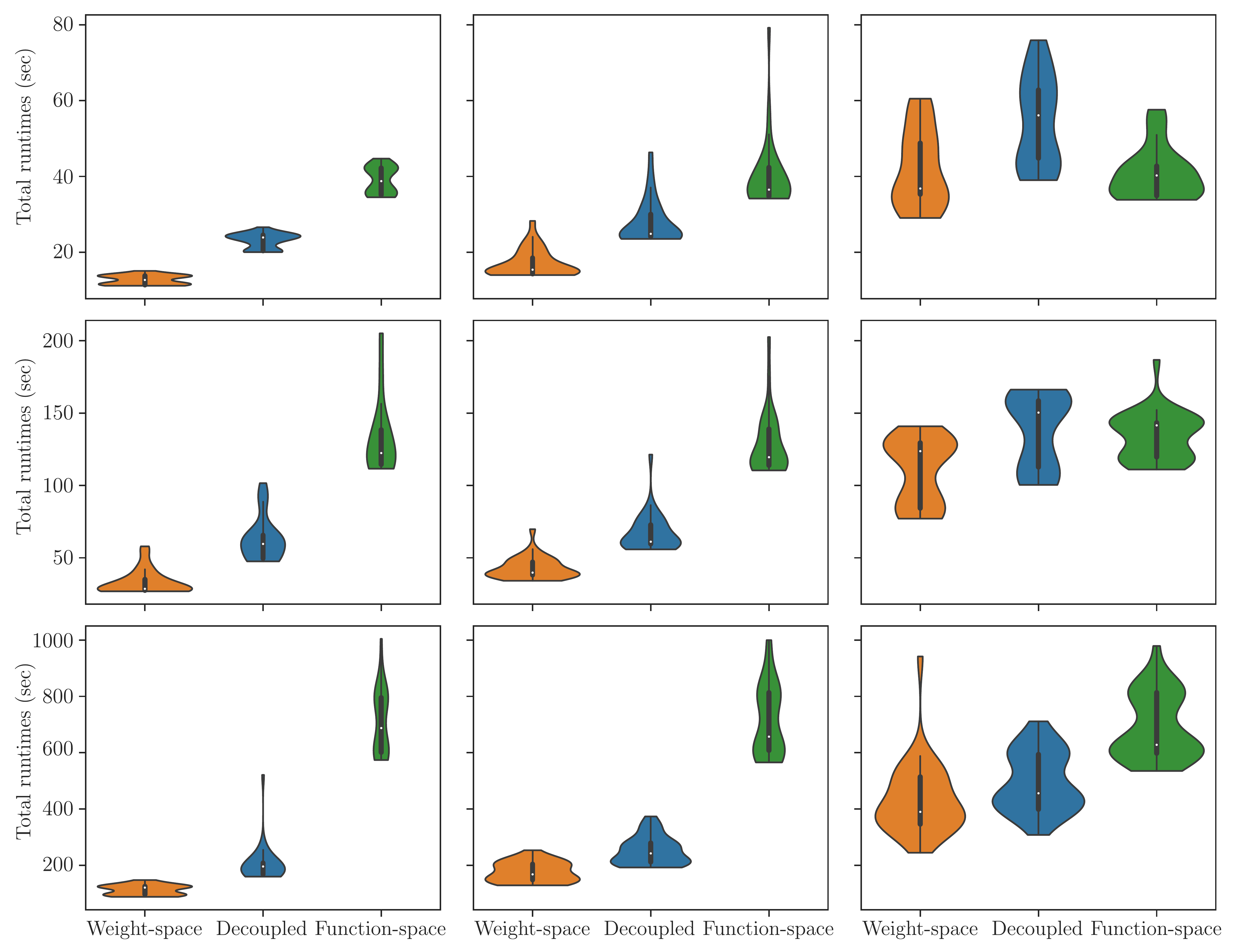}
\caption{Empirical distributions of per trial runtimes for parallel TS with different sampling strategies; subplots are 1-to-1 with those in Figure~\ref{fig:thompson_sampling_results_full}.}
\label{fig:thompson_sampling_runtimes_full}
\end{figure*}

\subsection{Dynamical systems}
\label{apdx:additional_experiments_dyn}
We investigated decoupled sampling's impact on (sequential) Monte Carlo methods' runtimes by using a sparse GP to simulate a simple dynamical system, the FitzHugh-Nagumo model neuron \cite{fitzhugh1961impulses,nagumo1962active} with diffusion coefficient $\m{\Sigma} = 0.01 \cdot \m{I}$. Training and simulation were both performed using a step size $\Delta t = 0.25$. 

During training, independent sparse GPs with $m=32$ shared inducing locations were fit to 3-dimensional inputs $\v{x}_{t} = [\v{s}_{t}, \v{c}_{t}]$, where $\v{s} \in [0, 1]^{2}$ denotes the (normalized) state vector at time $t$ and $\v{c} \in [0, 1]$ the coinciding (normalized) control input, with targets defined as the $i$-th element of the Euler-Maruyama transition vectors specified by \eqref{eqn:euler_maruyama}. Owing to the need to separate out signal from noise, the training set consisted of $10,000$ uniform random training points and training was performed using stochastic gradient descent.

At test time, a baseline was constructed by iteratively drawing drift vectors $f_{t+1} \given \v{f}_{1:t}$. At each iteration, the current input $\v{x}_{t}$ is added to the set of inducing locations $\m{Z}_{t+1} = \m{Z}_{t} \cup \{\v{x}_{t}\}$ and the $i$-th inducing distribution is augmented to incorporate the sampled drift as
\begin{align}
q^{(i)}_{t+1}(\v{u}) = \c{N}\left(
    \left[\begin{array}{c}
    \v{\mu}_{t}^{(i)} \\
    f_{t}^{(i)}
    \end{array}\right],
    \left[\begin{array}{cc}
    \m{\Sigma}_{t} - \v{v}\v{v}^\top & \m{0}\\
    \m{0} & 0
    \end{array}\right]
\right)
\end{align}
where $\v{v} = k_{t}(\v{x}_{t}, \m{Z}_{t}) k_{t}(\v{x}_{t}, \v{x}_{t})^{\nicefrac{-1}{2}}$ is defined in terms of the posterior covariance given the $m + t$ preceding inducing locations. When the inducing covariance is parameterized by its Cholesky factor, $\m{\Sigma}_{t+1}^{\nicefrac{1}{2}}$ can be directly computed via a rank-1 downdate \cite{gill1974methods,seeger2004low}. Since only the $m$-th leading principal submatrix of $\m{\Sigma}_{t+1}^{\nicefrac{1}{2}}$ needs to be modified (the remaining terms are all zero because $\v{f}_{t}$ is directly observed), this downdate incurs $\c{O}(m^{2})$ time complexity per iteration. In similar fashion, the prior covariance and its Cholesky factor may be maintained online. Here, however, as well as when computing posterior marginals, the matrices are no longer sparse, resulting in $\c{O}((m + t)^{2})$ cost per step. Overall, the iterative approach to unrolling scales cubically in the number of steps.

\end{document}